\documentclass{article}
\usepackage[dvipsnames]{xcolor}
\usepackage{times}
\usepackage{amsmath,amsbsy,amsfonts,amsthm, units}
\usepackage{graphicx}
\usepackage{tablefootnote}
\usepackage{url}
\usepackage{cases}
\usepackage{tikz}
\usetikzlibrary{calc,shapes}
\usepackage{subfigure}
\definecolor{DarkGreen}{rgb}{0.1,0.5,0.1}
\definecolor{DarkRed}{rgb}{0.5,0.1,0.1}
\definecolor{DarkBlue}{rgb}{0.1,0.1,0.5}
\usepackage{nicefrac}
\usepackage{algorithm}
\usepackage{algorithmic}
\usepackage{xspace}
\usepackage{wrapfig}
\usepackage{hyperref}

\newcommand{\cD}{\mathcal{D}}
\newcommand{\relu}{\textup{ReLu}}
\newcommand{\pp}[1]{^{(#1)}}
\newcommand{\gnorm}[1]{{\left\vert\kern-0.25ex\left\vert\kern-0.25ex\left\vert #1 
		\right\vert\kern-0.25ex\right\vert\kern-0.25ex\right\vert}}

\newtheorem{theorem}{Theorem}[section]

\newtheorem{lemma}[theorem]{Lemma}
\newtheorem{claim}[theorem]{Claim}

\newtheorem{assumption}[theorem]{Assumption}

\theoremstyle{definition}

\numberwithin{equation}{section}

\newcommand{\IGNORE}[1]{}

\newcommand\E{\mathbb{E}}

\DeclareMathOperator{\diag}{diag}

\newcommand\inner[1]{\langle #1 \rangle}

\newcommand{\Exp}{\mathop{\mathbb E}\displaylimits}

\newcommand{\cB}{\mathcal{B}}
\newcommand{\cT}{\mathcal{T}}

\newcommand{\allones}{\mathbf{1}}

\newcommand{\norm}[1]{\lVert#1\rVert}
\newcommand{\Norm}[1]{\left\lVert#1\right\rVert}
\newcommand{\Id}{I}

\newcommand{\mper}{\,.}
\newcommand{\mcom}{\,,}

\newcommand\R{\mathbb{R}}
\newcommand\C{\mathbb{C}}
\newcommand\N{\mathcal{N}}

\newcommand{\partialgrad}[2]{{\frac{\partial #1 }{\partial #2}}}

\newcommand{\trace}{\textup{tr}}

\let\epsilon=\varepsilon

\numberwithin{equation}{section}

\newcommand\MYcurrentlabel{xxx}

\newcommand{\MYstore}[2]{	\global\expandafter \def \csname MYMEMORY #1 \endcsname{#2}}

\newcommand{\MYload}[1]{	\csname MYMEMORY #1 \endcsname}

\newcommand{\MYnewlabel}[1]{	\renewcommand\MYcurrentlabel{#1}	\MYoldlabel{#1}}

\newcommand{\MYdummylabel}[1]{}

\newcommand{\torestate}[1]{		\let\MYoldlabel\label	\let\label\MYnewlabel	#1	\MYstore{\MYcurrentlabel}{#1}		\let\label\MYoldlabel}

\newcommand{\restatetheorem}[1]{		\let\MYoldlabel\label
	\let\label\MYdummylabel
	\begin{theorem*}[Restatement of \prettyref{#1}]
		\MYload{#1}
	\end{theorem*}
	\let\label\MYoldlabel
}

\newcommand{\restatelemma}[1]{		\let\MYoldlabel\label
	\let\label\MYdummylabel
	\begin{lemma*}[Restatement of \prettyref{#1}]
		\MYload{#1}
	\end{lemma*}
	\let\label\MYoldlabel
}

\newcommand{\restateprop}[1]{		\let\MYoldlabel\label
	\let\label\MYdummylabel
	\begin{proposition*}[Restatement of \prettyref{#1}]
		\MYload{#1}
	\end{proposition*}
	\let\label\MYoldlabel
}

\newcommand{\restatefact}[1]{		\let\MYoldlabel\label
	\let\label\MYdummylabel
	\begin{fact*}[Restatement of \prettyref{#1}]
		\MYload{#1}
	\end{fact*}
	\let\label\MYoldlabel
}

\newcommand{\restate}[1]{		\let\MYoldlabel\label
	\let\label\MYdummylabel
	\MYload{#1}
	\let\label\MYoldlabel
}

\def\shownotes{0}  \ifnum\shownotes=1
\newcommand{\authnote}[2]{$\ll$\textsf{\footnotesize #1 notes: #2}$\gg$}
\else
\newcommand{\authnote}[2]{}
\fi
\newcommand{\Tnote}[1]{{\color{blue}\authnote{Tengyu}{#1}}}

\title{Identity Matters in Deep Learning}
\author{Moritz Hardt \thanks{
	Google Brain. 
	\texttt{m@mrtz.org}}.  
	\and 
	Tengyu Ma\thanks{
Stanford University. 
\texttt{tengyu@stanford.edu}. Work performed at Google. }}

\begin{document}
	\maketitle
\begin{abstract}
An emerging design principle in deep learning is that each layer of a deep
artificial neural network should be able to easily express the identity
transformation. This idea not only motivated various normalization techniques,
such as \emph{batch normalization}, but was also key to the immense success of
\emph{residual networks}.

In this work, we put the principle of \emph{identity parameterization} on a
more solid theoretical footing alongside further empirical progress. We first
give a strikingly simple proof that arbitrarily deep linear residual networks
have no spurious local optima. The same result for linear feed-forward networks in
their standard parameterization is substantially more delicate.  Second, we
show that residual networks with ReLu activations have universal finite-sample
expressivity in the sense that the network can represent any function of its
sample provided that the model has more parameters than the sample size.

Directly inspired by our theory, we experiment with a radically simple
residual architecture consisting of only residual convolutional layers and
ReLu activations, but no batch normalization, dropout, or max pool. Our model
improves significantly on previous all-convolutional networks on the CIFAR10,
CIFAR100, and ImageNet classification benchmarks.

\end{abstract}
\section{Introduction}

Traditional convolutional neural networks for image classification, such as AlexNet (\cite{krizhevsky2012imagenet}), are parameterized in such a way that when all trainable weights are~$0$, a convolutional layer represents the $0$-mapping. Moreover, the weights are initialized symmetrically around~$0.$ This standard parameterization makes it non-trivial for a convolutional layer trained with stochastic gradient methods to preserve features that were already good. Put differently, such convolutional layers cannot easily converge to the identity transformation at training time.

This shortcoming was observed and partially addressed by \cite{DBLP:conf/icml/IoffeS15} through \emph{batch normalization}, i.e., layer-wise whitening of the input with a learned mean and covariance. But the idea remained somewhat implicit until \emph{residual networks} (\cite{he15deepresidual}; \cite{DBLP:conf/eccv/HeZRS16})  explicitly introduced a reparameterization of the convolutional layers such that when all trainable weights are~$0,$ the layer represents the identity function. Formally, for an input $x,$ each residual layer has the form $x+h(x),$ rather than $h(x).$ This simple reparameterization allows for much deeper architectures largely avoiding the problem of vanishing (or exploding) gradients. Residual networks, and subsequent architectures that use the same parameterization, have since then consistently achieved state-of-the-art results on various computer vision benchmarks such as CIFAR10 and ImageNet.

\subsection{Our contributions}

In this work, we consider identity parameterizations from a theoretical perspective, while translating some of our theoretical insight back into experiments. Loosely speaking, our first result underlines how identity parameterizations make optimization easier, while our second result shows the same is true for representation.

\paragraph{Linear residual networks.}
Since general non-linear neural networks, are beyond the reach of current
theoretical methods in optimization, we consider the case of deep
\emph{linear} networks as a simplified model. A linear network represents an
arbitrary linear map as a sequence of matrices $A_\ell\cdots A_2A_1.$ The
objective function is $\E\|y-A_\ell\cdots A_1x\|^2$, where $y=Rx$ for some
unknown linear transformation~$R$ and $x$ is drawn from a distribution. Such
linear networks have been studied actively in recent years as a stepping stone
toward the general non-linear case (see Section~\ref{sec:related}). 
Even though $A_\ell\cdots A_1$ is just a
linear map, the optimization problem over the factored variables
$(A_\ell,\dots,A_1)$ is non-convex.

In analogy with residual networks, we will instead parameterize the objective function as
\begin{equation}\label{eq:linear-resnet}
\min_{A_1,\dots,A_\ell}\E\|y-(I+A_\ell)\cdots(I+A_1)x\|^2\,.
\end{equation}
To give some intuition, when the depth~$\ell$ is large enough, we can hope that the target function~$R$ has a factored representation in which each matrix $A_i$ has small norm. Any symmetric positive semidefinite matrix~$O$ can, for example, be written as a product $O=O_\ell\cdots O_1,$ where each $O_i=O^{1/\ell}$ is very close to the identity for large~$\ell$ so that $A_i=O_i-I$ has small spectral norm. 
We first prove that an analogous claim is true for all linear transformations~$R$ with positive determinant\footnote{As will be discussed below Theorem~\ref{thm:existence}, it is without loss of generality to assume that the determinant of $R$ is positive. }.  Specifically, we prove that for every linear transformation~$R$ with $\det(R) > 0$,  there exists a global optimizer $(A_1,\dots, A_\ell)$ of~\eqref{eq:linear-resnet} such that for large enough depth~$\ell,$
\begin{equation}
\max_{1\le i\le \ell}\|A_i\|\le O(1/\ell).
\end{equation}
Here, $\|A\|$ denotes the spectral norm of~$A.$ The constant factor depends on the conditioning of~$R.$ We give the formal statement in Theorem~\ref{thm:existence}. The theorem has the interesting consequence that as the depth increases, smaller norm solutions exist and hence regularization may offset the increase in parameters.

Having established the existence of small norm solutions, our main result on
linear residual networks shows that the objective
function~\eqref{eq:linear-resnet} is, in fact, easy to optimize when all
matrices have sufficiently small norm. More formally, letting
$A=(A_1,\dots,A_\ell)$ and $f(A)$ denote the objective function
in~\eqref{eq:linear-resnet}, we can show that the gradients vanish only
when~$f(A)=0$ provided that $\max_i\|A_i\|\le O(1/\ell).$ See
Theorem~\ref{thm:main}. This result implies that linear residual networks
have no \emph{critical points} other than the global optimum. 
In contrast, for standard linear neural networks we only know, by work
of \cite{Kawaguchi} that these networks don't have local optima except the
global optimum, but it doesn't rule out other critical points. In fact,
setting $A_i=0$ will always lead to a bad critical point in the standard
parameterization.

\paragraph{Universal finite sample expressivity.}
Going back to non-linear residual networks with ReLU activations, we can ask:
How expressive are deep neural networks that are solely based on residual
layers with ReLU activations? To answer this question, we give a very simple
construction showing that such residual networks have perfect finite sample
expressivity. In other words, a residual network with ReLU activations can
easily express any functions of a sample of size $n,$ provided that it has
sufficiently more than $n$ parameters. Note that this requirement is easily
met in practice. On CIFAR 10 ($n=50000$), for example, successful residual
networks often have more than $10^6$ parameters. More formally, for a data set
of size $n$ with $r$ classes, our construction requires $O(n\log n + r^2)$ 
parameters. Theorem~\ref{thm:representation} gives the formal statement.

Each residual layer in our construction is of the form $x + V\mathrm{ReLU}(Ux),$ where $U$ and $V$ are linear transformations. These layers are significantly simpler than standard residual layers, which typically have two ReLU activations as well as two instances of batch normalization.

\paragraph{The power of all-convolutional residual networks.}
Directly inspired by the simplicity of our expressivity result, we experiment
with a very similar architecture on the CIFAR10, CIFAR100, and ImageNet data
sets. Our architecture is merely a chain of convolutional residual layers each
with a single ReLU activation, but without batch normalization, dropout, or
max pooling as are common in standard architectures. The last layer is a fixed
random projection that is not trained. In line with our theory, the
convolutional weights are initialized near~$0,$ using Gaussian noise mainly as
a symmetry breaker. The only regularizer is standard weight decay
($\ell_2$-regularization) and there is no need for dropout. Despite its
simplicity, our architecture reaches $6.38\%$ top-$1$ classification error on
the CIFAR10 benchmark (with standard data augmentation). This is competitive
with the best residual network reported in~\cite{he15deepresidual}, which
achieved $6.43\%$. Moreover, it improves upon the performance of the previous
best \emph{all-convolutional} network, $7.25\%$, achieved
by~\cite{2014arXiv1412.6806S}. Unlike ours, this previous all-convolutional
architecture additionally required dropout and a non-standard preprocessing
(ZCA) of the entire data set. Our architecture also improves
significantly upon~\cite{2014arXiv1412.6806S} on both Cifar100 and ImageNet.

\subsection{Related Work}
\label{sec:related}
Since the advent of residual networks~(\cite{he15deepresidual}; \cite{DBLP:conf/eccv/HeZRS16}), most state-of-the-art networks for image classification have adopted a residual parameterization of the convolutional layers. Further impressive improvements were reported by \cite{DBLP:journals/corr/HuangLW16a} with a variant of residual networks, called \emph{dense nets}. Rather than adding the original input to the output of a convolutional layer, these networks preserve the original features directly by concatenation. In doing so, dense nets are also able to easily encode an identity embedding in a higher-dimensional space. It would be interesting to see if our theoretical results also apply to this variant of residual networks.

There has been recent progress on understanding the optimization landscape of neural networks, though a comprehensive answer remains elusive. Experiments in~
\cite{2014arXiv1412.6544G} and~\cite{dauphin2014identifying} suggest that the training objectives have a limited number of bad local minima with large function values. Work by~\cite{choromanska2015loss} draws an analogy between the optimization landscape of neural nets and that of the spin glass model in physics (\cite{auffinger2013random}).   \cite{2016arXiv160508361S} showed that  $2$-layer neural networks have no bad \textit{differentiable} local minima, but they didn't prove that a good differentiable local minimum does exist. \cite{Baldi:1989:NNP:70359.70362} and \cite{Kawaguchi} show that linear neural networks have no bad local minima. In contrast, we show that the optimization landscape of deep linear residual networks has no bad \textit{critical} point, which is a stronger and more desirable property. Our proof is also notably simpler illustrating the power of re-parametrization for optimization. Our results also indicate that deeper networks may have more desirable optimization landscapes compared with shallower ones. 

\section{Optimization landscape of linear residual networks} \label{sec:criticalpoints}

\newcommand{\copt}{C_{\textup{opt}}}

Consider the problem of learning a linear transformation 
$R\colon \R^d\to \R^d$ from noisy measurements $y=Rx + \xi,$
where $\xi\in \N(0,\Id_d)$ is a $d$-dimensional spherical
Gaussian vector. Denoting by $\cD$ the distribution of the input data~$x,$ let
$\Sigma  = \Exp_{x\sim \cD}[xx^{\top}]$ be its covariance matrix. 

There are, of course, many ways to solve this classical problem, but our goal is
to gain insights into the optimization landscape of neural nets, and in particular,
residual networks. We therefore parameterize our learned model
by a sequence of weight matrices $A_1,\dots, A_{\ell}\in \R^{d\times d}$, 
\begin{align}
h_0 & =  x \mcom\qquad
h_{j}  = h_{j-1} + A_{j} h_{j-1} \mcom\qquad
\hat{y} = h_{\ell} \mper
\end{align}
Here $h_{1},\dots, h_{\ell-1}$ are the $\ell-1$ hidden layers and $\hat{y} =
h_{\ell}$ are the predictions of the learned model on input~$x.$ 
More succinctly, we have 
\begin{align}
\hat{y} = (\Id + A_{\ell})\dots (\Id + A_{1}) x\mper \nonumber
\end{align}
It is easy to see that this model can express any linear transformation~$R.$
We will use $A$ as a shorthand for all of the weight matrices, that is, the
$\ell\times d\times d$-dimensional tensor that contains $A_1,\dots, A_{\ell}$
as slices. Our objective function is the maximum likelihood estimator, 
\begin{align}
f(A, (x,y)) = \norm{\hat{y}-y}^2 = \norm{ (\Id + A_\ell)\dots (\Id + A_{1}) x- Rx -\xi}^2 \mper\label{eqn:objective-single-exp}
\end{align}
We will analyze the landscape of the \emph{population risk}, defined as, 
\begin{align}
f(A) := \Exp\left[f(A,(x,y))\right] \mper \nonumber
\end{align}

Recall that $\norm{A_i}$ is the spectral norm of $A_i$. We define the norm $\gnorm{\cdot}$ for the tensor $A$ as the maximum of the spectral norms of its slices, 
\begin{align}
\gnorm{A} := \max_{1\le i\le \ell} \norm{A_i}\mper\nonumber
\end{align}
The first theorem of this section states that the objective function $f$ has
an optimal solution with small $\gnorm{\cdot}$-norm, which is
\textit{inversely} proportional to the number of layers $\ell$. Thus,  when
the architecture is deep, we can shoot for fairly small norm solutions. We
define $\gamma := \max\{|\log \sigma_{\max}(R)|, |\log \sigma_{\min}(R)|\}$.
Here $\sigma_{\min}(\cdot), \sigma_{\max}(\cdot)$ denote the least and largest
singular values of $R$ respectively. 
\begin{theorem}\label{thm:existence}
Suppose $\ell \ge 3\gamma$ and $\det(R) > 0$. Then, there exists a global optimum solution $A^{\star}$ of the population risk $f(\cdot)$ with norm $$\gnorm{A^{\star}}\le (4\pi + 3\gamma)/\ell\mper$$  
\end{theorem}

We first note that the condition $\det(R) > 0$ is without loss of generality in the following sense. Given any linear transformation $R$ with negative determinant, we can effectively flip the determinant by augmenting the data and the label with an additional dimension: let $x' = [x, b]$ and $y' = [y,-b]$ , where $b$ is an independent random variable (say, from standard normal distribution),  and let $R' = \begin{bmatrix}
R & 0 \\
0 & -1
\end{bmatrix}$. Then, we have that $y' = R'x' + \xi$ and $\det(R') = -\det(R) > 0$.\footnote{When the dimension is odd, there is an easier way to see this -- flipping the label corresponds to flipping $R$, and we have $\det(-R) = -\det(R)$.}

Second, we note that here $\gamma$ should be thought of as a constant since if $R$ is too large (or
too small), we can scale the data properly so that $\sigma_{\min}(R)\le 1\le
\sigma_{\max}(R)$. Concretely, if $\sigma_{\max}(R)/\sigma_{\min}(R) =
\kappa$, then we can scaling for the outputs properly so that
$\sigma_{\min}(R) = 1/\sqrt{\kappa}$ and $\sigma_{\max}(R) = \sqrt{\kappa}$.
In this case, we have $\gamma = \log \sqrt{\kappa}$, which will remain a small
constant for fairly large condition number $\kappa$. We also point out that we
made no attempt to optimize the constant factors here in the analysis. The
proof of Theorem~\ref{thm:existence} is rather involved and is deferred to Section~\ref{sec:proof}. 

Given the observation of Theorem~\ref{thm:existence}, we restrict our attention to analyzing the landscape of $f(\cdot)$ in the set of $A$ with $\gnorm{\cdot}$-norm less than $\tau$,  
\begin{align}
\cB_{\tau} = \{A\in \R^{\ell\times d\times d}: \gnorm{A}\le \tau\}\mper\nonumber
\end{align}

Here using Theorem~\ref{thm:existence}, the radius $\tau$ should be thought of as on the order of $1/\ell$. Our main theorem in this section claims that there is no bad critical point in the domain $\cB_{\tau}$ for any $\tau < 1$. Recall that a critical point has vanishing gradient. 

\begin{theorem}\label{thm:main}
	For any $\tau < 1$, we have that any critical point $A$ of the objective function $f(\cdot)$ inside the 
	domain $\cB_{\tau}$ must also be a global minimum. 
\end{theorem}
Theorem~\ref{thm:main} suggests that it is sufficient for the optimizer to converge to critical points of the population risk, since all the critical points are also global minima. 

Moreover, in addition to Theorem~\ref{thm:main}, we also have that any $A$ inside the domain $\cB_{\tau}$ satisfies that 
\begin{align}
\Norm{\nabla f(A)}_F^2  \ge 4\ell(1-\tau)^{2\ell-2} \sigma_{\min}(\Sigma) (f(A)- \copt) \mper\label{eqn:10}
\end{align}
Here $\copt$ is the global minimal value of $f(\cdot)$ and $\norm{\nabla f(A)}_F$ denotes the euclidean norm\footnote{That is, $\norm{T}_F := \sqrt{\sum_{ijk}T_{ijk}^2}$. } of the $\ell\times d\times d$-dimensional tensor $\nabla f(A)$. Note that $\sigma_{\min}(\Sigma)$ denote the minimum singular value of $\Sigma$. 

Equation~\eqref{eqn:10} says that the gradient has fairly large norm compared
to the error, which guarantees convergence of the gradient descent to a global
minimum (\cite{2016arXiv160804636K}) if the iterates stay inside the domain
$\cB_{\tau},$ which is not guaranteed by Theorem~\ref{thm:main} by itself.

Towards proving Theorem~\ref{thm:main}, we start off with a simple claim that
simplifies the population risk. We use $\norm{\cdot}_F$ to denote the Frobenius
norm of a matrix, and $\langle A,B\rangle$ denotes the inner product of $A$ and $B$ in the standard basis (that is, $\langle A,B\rangle = \trace(A^\top B)$ where $\trace(\cdot)$ denotes the trace of a matrix.)

\begin{claim}\label{claim:1}
	In the setting of this section, we have, 
	\begin{align}
	f(A) = \Norm{((\Id+A_{\ell})\dots (\Id+A_{1})- R)\Sigma^{1/2}}_F^2 + C\mper\label{eqn:obj}
	\end{align}
	Here $C$ is a constant that doesn't depend on $A$, and $\Sigma^{1/2}$ denote the square root of $\Sigma$, that is, the unique symmetric matrix $B$ that satisfies $B^2 = \Sigma$. 
\end{claim}

\begin{proof}[Proof of Claim~\ref{claim:1}] Let $\trace(A)$ denotes the trace of the matrix $A$. 
	Let $E = (\Id + A_\ell)\dots (\Id + A_{1}) - R$. 
	Recalling the definition of $f(A)$ and using equation~\eqref{eqn:objective-single-exp}, we have
	\begin{align}
	f(A) & = \Exp\left[\norm{ Ex -\xi}^2 \right] \tag{by equation~\eqref{eqn:objective-single-exp} }\\
	& =\Exp\left[\norm{Ex}^2 + \norm{\xi}^2 - 2\inner{Ex,\xi}\right] \nonumber\\
	& = \Exp\left[\trace(Exx^\top E^{\top})\right] +  \Exp\left[\norm{\xi}^2\right] \tag{since $\Exp\left[\inner{Ex,\xi}\right] = \Exp\left[\inner{Ex,\Exp\left[\xi\vert x\right]}\right]=0$}\\
	& = \trace\left(E\Exp\left[xx^{\top}\right]E^{\top}\right) + C \tag{where $C= \Exp[\xi^2]$}\\
	& = \trace(E\Sigma E^{\top}) + C = \norm{E\Sigma^{1/2}}_F^2 + C\mper\tag{since $\Exp\left[xx^{\top}\right] = \Sigma$}
	\end{align}
\end{proof}

Next we compute the gradients of the objective function $f(\cdot)$ from straightforward matrix calculus. We defer the full proof to Section~\ref{sec:proof}. 
\begin{lemma}\label{lem:gradients} 	The gradients of $f(\cdot)$ can be written as, 
\begin{align}
\partialgrad{f}{A_i} = 2(\Id+ A_{i+1}^{\top})\dots (\Id + A_{\ell}^{\top})E \Sigma(\Id + A_{1}^{\top})\dots (\Id + A_{i-1}^{\top})\mcom\label{eqn:grads}
\end{align}
where 
$E = (\Id+A_\ell)\dots (\Id+A_{1})- R$\mper
\end{lemma}

Now we are ready to prove Theorem~\ref{thm:main}. The key observation is that
each matric $A_j$ has small norm and cannot cancel the identity
matrix. Therefore, the gradients in equation~\eqref{eqn:grads} is a product
of non-zero matrices, except for the error matrix $E$. Therefore, if
the gradient vanishes, then the only possibility is that the matrix~$E$
vanishes, which in turns implies $A$ is an optimal solution. 

\begin{proof}[Proof of Theorem~\ref{thm:main}]
Using Lemma~\ref{lem:gradients}, we have, 
	\begin{align}
	\Norm{\partialgrad{f}{A_i}}_F & = 2\Norm{ (\Id+ A_{i+1}^{\top})\dots (\Id + A_{\ell}^{\top})E\Sigma (\Id + A_{1}^{\top})\dots (\Id + A_{i-1}^{\top}) }_F \tag{by Lemma~\ref{lem:gradients}}\\
	& \ge 2\prod_{j\neq i} \sigma_{\min}(\Id+A_i^{\top})\cdot \sigma_{\min}(\Sigma)\norm{E}_F \tag{by Claim~\ref{claim:sigmamin}} \\
	& \ge 2(1-\tau)^{\ell-1} \sigma_{\min}(\Sigma^{1/2})\norm{E\Sigma^{1/2}}_F\mper \tag{since $\sigma_{\min}(\Id +A )\ge 1-\norm{A}$}
	\end{align}
	It follows that 
	\begin{align}
	\Norm{\nabla f(A)}_F^2 & = \sum_{i=1}^{\ell}\Norm{\partialgrad{f}{A_i}}_F^2  \ge 4\ell (1-\tau)^{2(\ell-1)} \sigma_{\min}(\Sigma)\norm{E\Sigma^{1/2}}_F^2 \nonumber\\
	 & = 4\ell(1-\tau)^{2(\ell-1)} \sigma_{\min}(\Sigma)(f(A)-C) \tag{by the definition of $E$ and Claim~\ref{claim:1}}\\
	 & \ge 4\ell(1-\tau)^{2(\ell-1)} \sigma_{\min}(\Sigma)(f(A)-\copt)\mper\tag{since $\copt :=\min_A f(A)\ge C$ by Claim~\ref{claim:1}}\nonumber
	\end{align}
	Therefore we complete the proof of equation~\eqref{eqn:10}. Finally,  if $A$ is a critical point, namely, $\nabla f(A) =0$, then by equation~\eqref{eqn:10}  we have that $f(A) = \copt$. That is, $A$ is a global minimum. 
\end{proof}

\section{Representational Power of the Residual Networks}
\label{sec:representation}

\newcommand{\ReLU}{ReLU~}

In this section we characterize the finite-sample expressivity of residual
networks.
We consider a residual layers with a single \ReLU activation and no batch
normalization.  The basic residual building block is a function $\cT_{U,V,
s}(\cdot): \R^k\rightarrow \R^k$ that is parameterized by two weight matrices
$U\in \R^{k \times k},V\in \R^{k\times k}$ and a bias vector $s\in \R^k$, 
\begin{align}
\cT_{U,V, s}(h) = V\relu(Uh + s)\mper \label{eqn:block}
\end{align}
A residual network is composed of a sequence of such residual blocks.
In comparison with the full pre-activation architecture
in~\cite{DBLP:conf/eccv/HeZRS16}, we remove two batch normalization layers and
one \ReLU layer in each building block. 

We assume the data has $r$ labels, encoded as~$r$ standard basis vectors in
$\R^r$, denoted by $e_1,\dots, e_r$. We have $n$ training examples $(x\pp{1},
y\pp{1}),\dots, (x\pp{n},y\pp{n})$, where $x\pp{i}\in \R^d$ denotes the $i$-th
data and $y\pp{i}\in \{e_1,\dots, e_r\}$ denotes the $i$-th label. Without
loss of generality we assume the data are normalized so that $x\pp{i}=1.$
We also make the mild assumption that no two data points are very close to each
other. 
\begin{assumption}\label{ass:data}
We assume that for every $1\le i < j \le n$, we have $\norm{x\pp{i}-x\pp{j}}^
2 \ge \rho $ for some absolute constant $\rho > 0.$
\end{assumption}
Images, for example, can always be imperceptibly perturbed in pixel
space so as to satisfy this assumption for a small but constant~$\rho.$

Under this mild assumption, we prove that residual networks have the
power to express any possible labeling of the data as long as the number of 
parameters is a logarithmic factor larger than~$n$.  

\begin{theorem}\label{thm:representation}
Suppose the training examples satisfy Assumption~\ref{ass:data}.  Then,
there exists a residual network $N$ (specified below) with $O(n\log n +r^2)$
parameters that perfectly expresses the training data, i.e., for all
$i\in\{1,\dots,n\},$ the network $N$ maps $x\pp{i}$ to $y\pp{i}.$
\end{theorem}
It is common in practice that $n>r^2,$ as is for example the case for the
Imagenet data set where $n>10^6$ and $r=1000.$

We construct the following residual net using the building blocks of the form $\cT_{U,V,s}$ as defined in equation~\eqref{eqn:block}. The network consists of $\ell+1$ hidden layers $h_0,\dots, h_{\ell}$, and the output is denoted by $\hat{y}\in \R^r$. The first layer of weights matrices $A_0$ maps the $d$-dimensional input to a $k$-dimensional hidden variable $h_0$. Then we apply $\ell$ layers of building block $\cT$ with weight matrices $A_{j},B_j\in \R^{k\times k}$. Finally, we apply another layer to map the hidden variable $h_{\ell}$ to the label $\hat{y}$ in $\R^k$. Mathematically, we have
\begin{align}
h_{0} & = A_0x\mcom  \nonumber\\
h_{j} & = h_{j-1}+ \cT_{A_j, B_j, b_j}(h_{j-1}), \quad \forall j\in \{1,\dots,\ell\}\nonumber\\
\hat{y} & = \cT_{A_{\ell+1},B_{\ell+1}, s_{\ell+1}}(h_{\ell}) \mper\nonumber
\end{align} 
We note that here $A_{\ell+1}\in \R^{k\times r}$ and $B_{\ell+1}\in \R^{r\times r}$ so that the dimension is compatible. We assume the number of labels $r$ and the input dimension $d$ are both smaller than $n$, which is safely true in practical applications.\footnote{In computer vision, typically $r$ is less than $10^3$ and $d$ is less than $10^5$ while $n$ is larger than $10^6$} The hyperparameter $k$ will be chosen to be $O(\log n)$ and the number of layers is chosen to be $\ell = \lceil n/k\rceil $. Thus, the first layer has $dk$ parameters, and each of the middle $\ell$ building blocks contains $2k^2$ parameters and the final building block has $kr+r^2$ parameters. Hence, the total number of parameters is $O(k d+ \ell k^2 + rk + r^2) = O(n\log n + r^2)$. 

Towards constructing the network $N$ of the form above that fits the data, we first take a random matrix $A_0\in \R^{k\times d}$ that maps all the data points $x\pp{i}$ to vectors $h_0\pp{i} := A_0 x\pp{i}$. Here we will use $h_j\pp{i}$ to denote the $j$-th layer of hidden variable of the $i$-th example. 
By Johnson-Lindenstrauss Theorem (\cite{johnson1984extensions}, or see ~\cite{wiki:JL}), with good probability, the resulting vectors $h_0\pp{i}$'s remain to satisfy Assumption~\ref{ass:data} (with slightly different scaling and larger constant $\rho$), that is, any two vectors $h_0\pp{i}$ and $h_0\pp{j}$ are not very correlated. 

Then we construct $\ell$ middle layers that maps $h_0\pp{i}$ to $h_{\ell}\pp{i}$ for every $i \in \{1,\dots, n\}$. These vectors $h_{\ell}\pp{i}$ will  clustered into $r$ groups according to the labels, though they are in the $\R^k$ instead of in $\R^r$ as desired. Concretely, we design this cluster centers by picking $r$ random unit vectors $q_1, \dots, q_r$ in $\R^k$. We view them as the surrogate label vectors in dimension $k$ (note that $k$ is potentially much smaller than $r$). In high dimensions (technically, if $k  > 4\log r$) random unit vectors $q_1,\dots, q_r$ are pair-wise uncorrelated with inner product less than $<0.5$.  We associate the $i$-th example with the target  surrogate label vector $v\pp{i}$  defined as follows, 
\begin{align}
\textup{if } y\pp{i} = e_j,  \textup{ then } v\pp{i} = q_j \mper \label{eqn:def-v}\end{align}

Then we will construct the matrices $(A_1,B_1),\dots, (A_{\ell},B_{\ell})$ such that the first $\ell$ layers of the network maps vector $h_0\pp{i}$ to the surrogate label vector $v\pp{i}$. Mathematically, we will construct $(A_1,B_1),\dots, (A_{\ell},B_{\ell})$ such that 
\begin{align}
\forall i \in \{1,\dots, n\}, h_{\ell}\pp{i} = v\pp{i}\mper\label{eqn:llayer}
\end{align}

Finally we will construct the last layer $\cT_{A_{\ell+1},B_{\ell+1},b_{\ell+1}}$ so that it maps the vectors $q_1,\dots, q_r\in \R^k$ to $e_1,\dots, e_r\in \R^r$, 
\begin{align}
\forall j\in \{1,\dots, r\}, \cT_{A_{\ell+1},B_{\ell+1},b_{\ell+1}}(q_j) = e_j\mper\label{eqn:lastlayer}
\end{align}Putting these together, we have that by  the definition~\eqref{eqn:def-v} and equation~\eqref{eqn:llayer},  for every $i$, if the label is $y\pp{i}$ is $e_j$, then $h_\ell\pp{i}$ will be $q_j$. Then by equation~\eqref{eqn:lastlayer}, we have that $\hat{y}\pp{i} =  \cT_{A_{\ell+1},B_{\ell+1},b_{\ell+1}}(q_j) = e_j$. Hence we obtain that $\hat{y}\pp{i} =y\pp{i} $.   
The key part of this plan is the construction of the middle $\ell$ layers of weight matrices so that $h_{\ell}\pp{i} = v\pp{i}$. We encapsulate this into the following informal lemma. The formal statement and the full proof is deferred to Section~\ref{sec:proof:representation}.

\begin{lemma}[Informal version of Lemma~\ref{lem:induction}]\label{lem:induction_informal}
	In the setting above, for (almost) arbitrary vectors $h_0\pp{1},\dots, h_0\pp{n}$ and $v\pp{1},\dots,v\pp{n}\in \{q_1,\dots, q_r\}$,  there exists weights matrices $(A_1,B_1),\dots, (A_{\ell},B_{\ell})$, such that, 
	\begin{align}
\forall i\in \{1,\dots, n\}	, ~~~ h_{\ell}\pp{i} = v\pp{i}\mper\nonumber
	\end{align}
\end{lemma}

We briefly sketch the proof of the Lemma to provide intuitions, and defer the full proof to Section~\ref{sec:proof:representation}. The operation that each residual block applies to the hidden variable can be abstractly written as, 
\begin{align}
\hat{h} \rightarrow h + \cT_{U,V,s}(h)\mper\label{eqn:op}
\end{align}
where $h$ corresponds to the hidden variable before the block and $\hat{h}$ corresponds to that after. We claim that 
for an (almost) arbitrary sequence of vectors $h\pp{1},\dots,h\pp{n}$, there exists $\cT_{U,V,s}(\cdot)$ such that operation~\eqref{eqn:op} transforms $k$ vectors of $h\pp{i}$'s to an arbitrary set of other $k$ vectors that we can freely choose, and maintain the value of the rest of $n-k$ vectors. Concretely, for any subset $S$ of size $k$, and any desired vector $v\pp{i}  (i\in S)$, there exist $U,V,s$ such that 
\begin{align}
v\pp{i} & = h\pp{i} + \cT_{U,V,s}(h\pp{i})~~\forall i\in S \nonumber\\
h\pp{i} & = h\pp{i} + \cT_{U,V,s}(h\pp{i})~~ \forall i\not\in S \label{eqn:14}
\end{align}
This claim is formalized in Lemma~\ref{lem:building_block}. We can use it repeatedly to construct $\ell$ layers of building blocks, each of which transforms a subset of $k$ vectors in $\{h_0\pp{1},\dots, h_0\pp{n}\}$ to the corresponding vectors in $\{v\pp{1},\dots, v\pp{n}\}$, and maintains the values of the others. Recall that we have $\ell = \lceil n/k\rceil$ layers and therefore after $\ell$ layers,  all the vectors $h_0\pp{i}$'s are transformed to $v\pp{i}$'s, which complete the proof sketch. \qed

\section{Power of all-convolutional residual networks}

Inspired by our theory, we experimented with all-convolutional residual
networks on standard image classification benchmarks. 

\subsection{CIFAR10 and CIFAR100}
Our architectures for CIFAR10 and CIFAR100 are identical except for the final
dimension corresponding to the number of classes $10$ and $100$, respectively. In
Table~\ref{table:cifar10-architecture}, we outline our architecture. Each
\emph{residual block} has the form $x+C_2(\mathrm{ReLU}(C_1x)),$ where
$C_1,C_2$ are convolutions of the specified dimension (kernel width, kernel
height, number of input channels, number of output channels). The second
convolution in each block always has stride~$1$, while the first may have
stride~$2$ where indicated. In cases where transformation is not
dimensionality-preserving, the original input $x$ is adjusted using averaging
pooling and padding as is standard in residual layers.

We trained our models with the Tensorflow framework, using a momentum
optimizer with momentum $0.9,$ and batch size is~$128$.  All convolutional
weights are trained with weight decay~$0.0001.$ 
The initial learning rate is $0.05,$ which drops by a factor $10$ and $30000$
and $50000$ steps.  The model reaches peak performance at
around $50k$ steps, which takes about $24h$  on a single NVIDIA Tesla K40 GPU.
Our code can be easily derived from an open source
implementation\footnote{\url{https://github.com/tensorflow/models/tree/master/resnet}}
by removing batch normalization, adjusting the residual components and model
architecture. An important departure from the code is that we initialize a
residual convolutional layer of kernel size $k\times k$ and $c$ output channels using a
random normal initializer of standard deviation $\sigma=1/k^2c,$ rather than
$1/k\sqrt{c}$ used for standard convolutional layers. This substantially
smaller weight initialization helped training, while not affecting
representation.

A notable difference from standard models is that the last layer is not
trained, but simply a fixed random projection. On the one hand, this slightly
improved test error (perhaps due to a regularizing effect). On the other hand,
it means that the only trainable weights in our model are those of the
convolutions, making our architecture ``all-convolutional''. 

\begin{table}[ht]
\caption{Architecture for CIFAR10/100 ($55$ convolutions, $13.5$M parameters)}
\centering
\begin{tabular}{|l|c|l|}
\hline
variable dimensions & initial stride & description \\
\hline
$3\times 3\times 3\times 16$ &  $1$ & 1 standard conv \\
$3\times 3\times 16\times 64$ &  $1$ & 9 residual blocks \\
$3\times 3\times 64\times 128$ &  $2$ & 9 residual blocks \\
$3\times 3\times 128\times 256$ &  $2$ & 9 residual blocks \\
\hline
-- & -- & $8\times 8$ global average pool \\
$256\times\mathtt{num\_classes}$ & -- & random projection (not trained)\\
\hline
\end{tabular}
\label{table:cifar10-architecture}
\end{table}

An interesting aspect of our model is that despite its massive size of $13.59$
million trainable parameters, the model does not seem to overfit too quickly
even though the data set size is $50000.$ In contrast, we found it difficult
to train a model with batch normalization of this size without significant
overfitting on CIFAR10.

\begin{figure}
\includegraphics[width=0.49\textwidth]{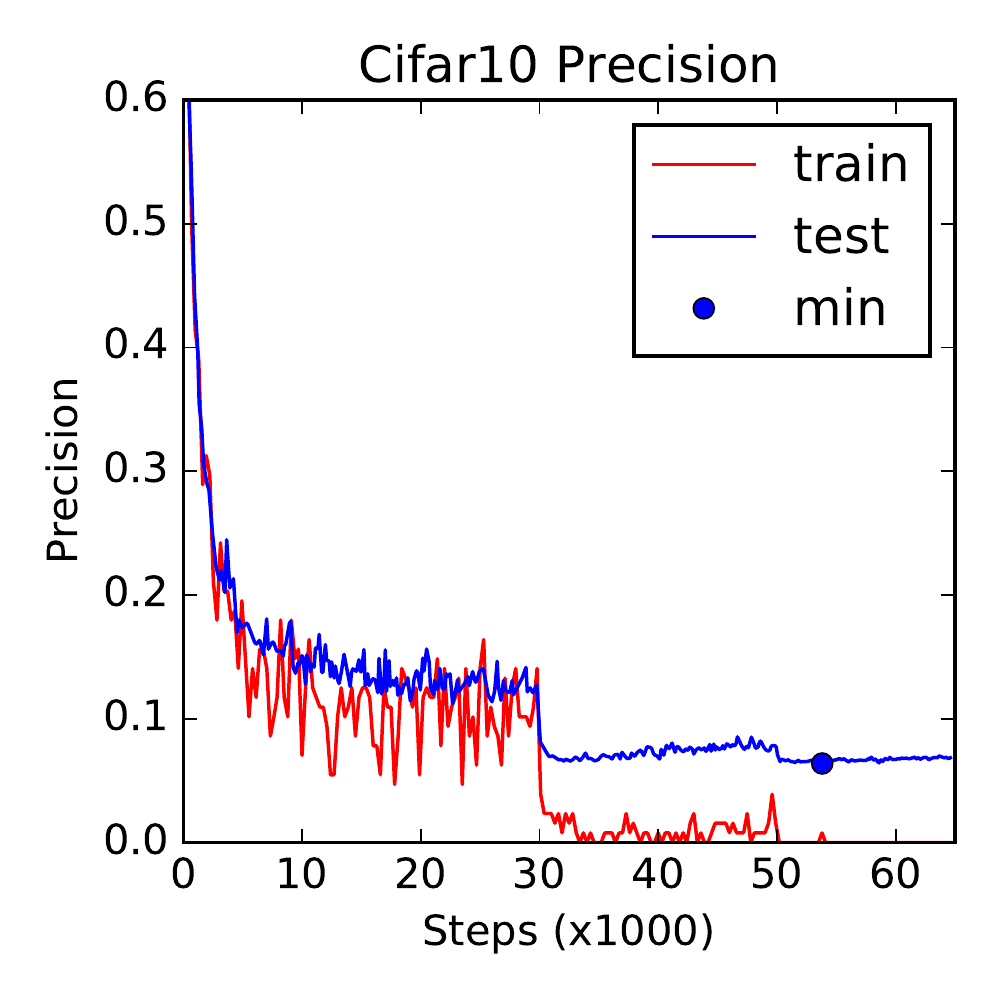}
\includegraphics[width=0.49\textwidth]{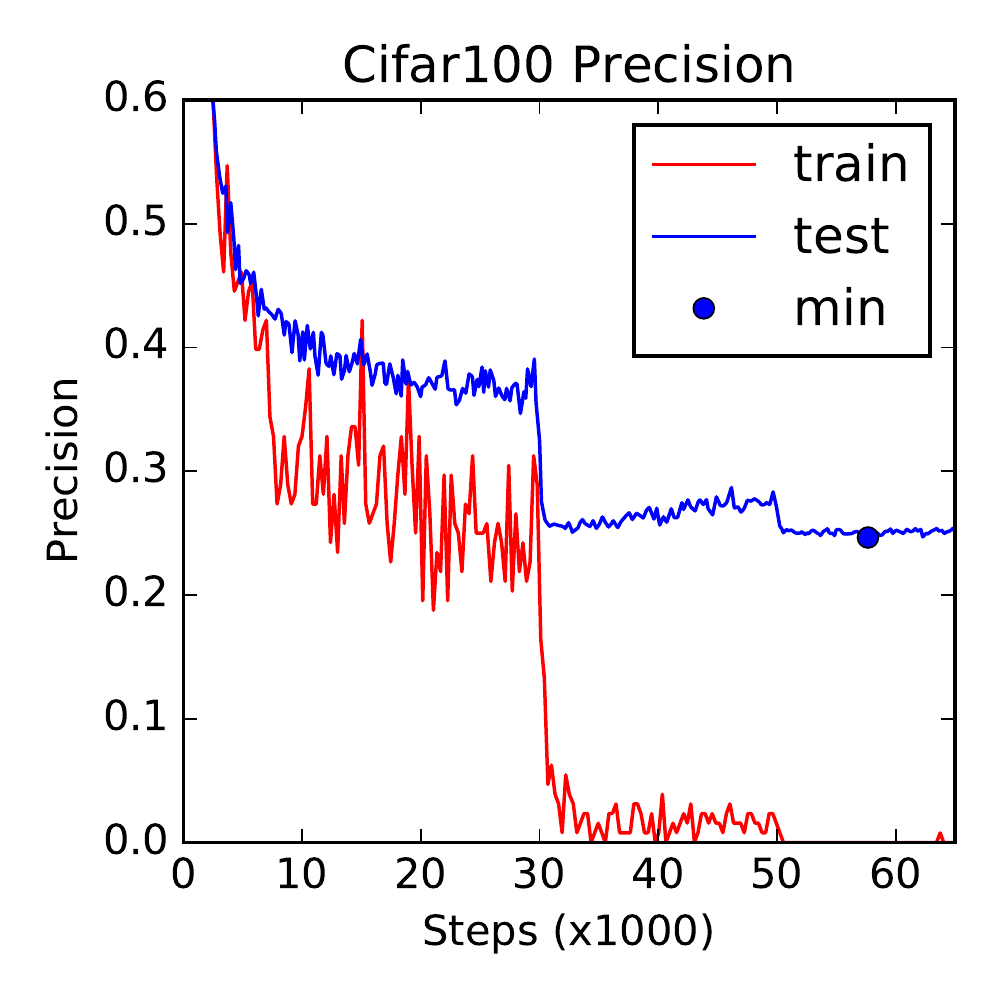}
\caption{Convergence plots of best model for CIFAR10 (left) and CIFAR (100)
right. One step is a gradient update with batch size $128$.}
\end{figure}

Table~\ref{table:comparison} summarizes the top-$1$ classification error of our models
compared with a non-exhaustive list of previous works, restricted
to the best previous all-convolutional result by \cite{2014arXiv1412.6806S}, the first residual
results~\cite{he15deepresidual}, and state-of-the-art results on CIFAR by
\cite{DBLP:journals/corr/HuangLW16a}.
All results are with standard data augmentation.

\begin{table}[ht]
\caption{Comparison of top-$1$ classification error on different benchmarks}
\centering
\begin{tabular}{|l|c|c|c|p{37mm}|}
\hline
Method & CIFAR10 & CIFAR100 & ImageNet & remarks\\
\hline
All-CNN & $7.25$ & $32.39$ & $41.2$ & all-convolutional, dropout \newline extra data processing\\
\hline
Ours & $6.38$ & $24.64$ & $35.29$ & all-convolutional\\
\hline
ResNet & $6.43$ & $25.16$ & $19.38$ & \\
\hline 
DenseNet & $3.74$ & $19.25$ & N/A & \\
\hline  
\end{tabular}
\label{table:comparison}
\end{table}

\subsection{ImageNet}

The ImageNet ILSVRC 2012 data set has $1,281,167$ data points with $1000$
classes. Each image is resized to $224\times 224$ pixels with $3$ channels.
We experimented with an all-convolutional variant of the
$34$-layer network in~\cite{he15deepresidual}. The original model achieved
$25.03\%$ classification error. 
Our derived model has $35.7M$ trainable parameters.
We trained the model with a momentum optimizer (with
momentum $0.9$) and a learning rate schedule that decays by a factor of $0.94$
every two epochs, starting from the initial learning rate~$0.1.$ Training was
distributed across $6$ machines updating asynchronously. Each machine was
equipped with $8$ GPUs (NVIDIA Tesla K40) and used batch size~$256$ 
split across the $8$ GPUs so that each GPU updated with batches of size~$32.$

In contrast to the situation with CIFAR10 and CIFAR100, on ImageNet our
all-convolutional model performed significantly worse than its original
counterpart. Specifically, we experienced a significant amount of
\emph{underfitting} suggesting that a larger model would likely perform
better.

Despite this issue, our model still reached~$35.29\%$ top-$1$ classification
error on the test set ($50000$ data points), and $14.17\%$ top-$5$ test error
after $700,000$ steps (about one week of training).  While no longer
state-of-the-art, this performance is significantly better than the $40.7\%$
reported by \cite{krizhevsky2012imagenet}, as well as the best
all-convolutional architecture by \cite{2014arXiv1412.6806S}. We believe it is
quite likely that a better learning rate schedule and hyperparameter settings
of our model could substantially improve on the preliminary performance
reported here.

\section{Conclusion}

Our theory underlines the importance of identity parameterizations when
training deep artificial neural networks. An outstanding open problem is to
extend our optimization result to the non-linear case where each residual has
a single ReLU activiation as in our expressivity result. We conjecture that a
result analogous to Theorem~\ref{thm:main} is true for the general non-linear
case. Unlike with the standard parameterization, we see no fundamental
obstacle for such a result.

We hope our theory and experiments together help simplify the state of deep
learning by aiming to explain its success with a few fundamental principles,
rather than a multitude of tricks that need to be delicately combined. We
believe that much of the advances in image recognition can be achieved with
residual convolutional layers and ReLU activations alone. This could lead to
extremely simple (albeit deep) architectures that match the state-of-the-art
on all image classification benchmarks.

\vspace{0.2in}
\noindent{\bf Acknowledgment:} We thank Jason D. Lee, Qixing Huang, and Jonathan Shewchuk for helpful discussions and kindly pointing out errors in earlier versions of the paper. We also thank Jonathan Shewchuk for suggesting an improvement of equation~\eqref{eqn:10} that is incorporated into the current version. Tengyu Ma would like to thank the support by Dodds Fellowship and Siebel Scholarship. 

\bibliography{deep_learning}
\appendix
\section{Missing Proofs in Section~\ref{sec:criticalpoints}}\label{sec:proof}

In this section, we give the complete proofs for Theorem~\ref{thm:existence} and Lemma~\ref{lem:gradients}, which are omitted in Section~\ref{sec:criticalpoints}. 
\subsection{Proof of Theorem~\ref{thm:existence}}
It turns out the proof will be significantly easier if $R$ is assumed to be a \textit{symmetric positive semidefinite} (PSD) matrix, or if we allow the variables to be complex matrices. Here we first give a proof sketch for the first special case. The readers can skip it and jumps to the full proof below. We will also prove stronger results, namely, $\gnorm{A^{\star}}\le 3\gamma/\ell$, for the special case.

When $R$ is PSD, it can be diagonalized by orthonormal matrix $U$ in the sense that $R = UZU^{\top}$, where $Z =\diag(z_1,\dots, z_d)$ is a diagonal matrix with non-negative diagonal entries $z_1,\dots, z_d$. Let $A^{\star}_{1}=\dots = A^{\star}_{\ell} = U\diag(z_i^{1/{\ell}})U^{\top} - \Id$, then we have 
\begin{align}
(\Id+A^{\star}_{\ell})\cdots (\Id+A^{\star}_{1})  & = (U\diag(z_i^{1/\ell})U^{\top})^{\ell}  = U\diag(z_i^{1/{\ell}})^{\ell} U \tag{since $U^{\top}U = \Id$} \\
& = UZU^{\top} = R\mper \nonumber
\end{align} 

We see that the network defined by $A^{\star}$ reconstruct the transformation $R$, and therefore it's a global minimum of the population risk (formally see Claim~\ref{claim:1} below). Next, we verify that each of the $A^{\star}_j$ has small spectral norm: 
\begin{align}
\norm{A^{\star}_j} & = \norm{\Id-U\diag(z_i^{1/\ell})U^{\top})} = \norm{U(\Id-\diag(z_i)^{1/\ell})U^{\top}} = \norm{\Id-\diag(z_i)^{1/\ell}} \tag{since $U$ is orthonormal}\\
& = \max_i |z_i^{1/\ell}-1|\mper \label{eqn:15}
\end{align}
Since $\sigma_{\min}(R)\le z_i\le \sigma_{\max}(R)$, we have $\ell \ge 3\gamma\ge |\log z_i|$. It follows that 
\begin{align}
|z_i^{1/\ell} - 1| = |e^{(\log z_i)/\ell} - 1| \le 3|(\log z_i)/\ell|\le 3\gamma/\ell \mper\tag{since $|e^{x}-1|\le 3|x|$ for all $|x|\le 1$}
\end{align}
Then using equation~\eqref{eqn:15} and the equation above,  we have that $\gnorm{A}\le \max_j \norm{A^{\star}_j}\le 3\gamma/\ell$, which  completes the proof for the special case. 

Towards fully proving the Theorem~\ref{thm:existence}, we start with the following Claim: 

\begin{claim}\label{claim:2d}
	Suppose $Q\in \R^{2\times 2}$ is an orthonormal matrix. Then for any integer $q$, there exists matrix $W_1,\dots, W_q\in \R^{2\times 2}$ and a diagonal matrix $\Lambda$ satisfies that (a) $Q = W_1\dots W_q\Lambda$ and $\norm{W_j-\Id}\le \pi/q$, (b) $\Lambda$ is an diagonal matrix with $\pm 1$ on the diagonal, and (c) If $Q$ is a rotation then $\Lambda = \Id$. 
\end{claim}

\begin{proof}
	We first consider the case when $Q$ is a rotation. 	Each rotation matrix can be written as $T(\theta) := \begin{bmatrix}
	\cos \theta & -\sin \theta \\
	\sin \theta & \cos\theta
	\end{bmatrix}$. Suppose $Q = T(\theta)$. Then we can take $W_1=\dots = W_q = T(\theta/q)$ and $\Lambda = \Id$. 
	We can verify that 
	\begin{align}
	\norm{W_j - \Id} \le \theta/q. \nonumber\end{align}
	
	Next, we consider the case when $Q$ is a reflection. Then we have that $Q$ can be written as $Q =  T(\theta)\cdot \diag(-1,1) $,  where $\diag(-1,1)$ is the reflection with respect to the $y$-axis. Then we can take $W_1=\dots = W_q = T(\theta/q)$ and $\Lambda =  \diag(-1,1)$ and complete the proof. 
\end{proof}
Next we give the formal full proof of Theorem~\ref{thm:existence}. The main idea is to reduce to the block diagonal situation and to apply the Claim above.  

\begin{proof}[Proof of Theorem~\ref{thm:existence}]			
	Let $R= UKV^{\top}$ be the singular value decomposition of $R$, where $U$,$V$ are two orthonormal matrices and $K$ is a diagonal matrix with nonnegative entries on the diagonal. Since $\det(R) = \det(U)\det(K)\det(V)> 0$ and $\det(K)  > 0$, we can flip $U,V$ properly so that $\det(U) =\det(V) =1 $. Since $U$ is a normal matrix (that is, $U$ satisfies that $UU^{\top} = U^{\top}U$), by Claim~\ref{cliam:block-diagonal}, we have that $U$ can be block-diagonalized by orthonormal matrix $S$ into $U = SDS^{-1}$, where $D = \diag(D_1,\dots, D_{m})$ is a real block diagonal matrix with each block $D_i$ being of size at most $2\times 2$. Using Claim~\ref{claim:2d}, we have that for any $D_i$, there exists $W_{i,1},\dots, W_{i,q}, \Lambda_i$ such that \begin{align}D_i = W_{i,1}\dots W_{i,q} \Lambda_i\label{eqn:100}\end{align} and $\norm{W_{i,j}-\Id}\le \pi/q$. 
	Let  $\Lambda = \diag(\Lambda_1,\dots, \Lambda_m)$ and $W_j = \diag(W_{1,j},\dots W_{m,j})$. We can rewrite equation~\eqref{eqn:100} as 
	\begin{align}
		D = W_1\dots W_q \Lambda. 	\label{eqn:103}
	\end{align}
	Moreover, we have that $\Lambda$ is a diagonal matrix with $\pm 1$ on the diagonal. Since $W_{i,j}$'s are orthonormal matrix with determinant 1, we have $\det(\Lambda) = \det(D) = \det(U)=1$. That is, $\Lambda$ has an even number of $-1$'s on the diagonal. Then we can group the $-1$'s into $2\times 2$ blocks. Note that $\begin{bmatrix}
	-1 & 0\\
	0 & -1
	\end{bmatrix}$ is the rotation matrix $T(\pi)$. Thus we can write $\Lambda$ as a concatenation of $+1$'s on the diagonal and block $T(\pi)$. Then applying Claim~\ref{claim:2d} (on each of the block $T(\pi)$), we obtain that there are $W_1',\dots, W_q'$ such that 
	\begin{align}
	\Lambda = W_1'\dots W_q'\label{eqn:101}
	\end{align}
	where $\norm{W_{j}'-\Id}\le \pi/q$. 
	Thus using equation~\eqref{eqn:103} and~\eqref{eqn:10}, we obtain that 
	\begin{align}
	U = SDS^{-1} = SW_1S^{-1}\cdots SW_qS^{-1} \cdot SW_1'S^{-1} \cdots SW_{q}'S^{-1}\mper\nonumber\end{align}
	
      Moreover, we have that for every $j$, $\norm{SW_jS^{-1}-\Id} = \norm{S(W_j-\Id)S^{-1}} = \norm{W_j-\Id} \le\pi/q$, because $S$ is an orthonormal matrix. The same can be proved for $W_j'$. Thus let $B_j = SW_jS^{-1}-\Id $ for $j\le q$ and $B_{j+q} = SW_{j}'S^{-1}-\Id$, and we can rewrite, 
      \begin{align}
      U = (\Id+B_1)\dots (\Id + B_q)\mper\nonumber
      \end{align}
      
      We can deal with $V$ similarly by decomposing $V$ into $2q$ matrices that are $\pi/q$ close to identity matrix, 
      \begin{align}
      V^{\top} = (\Id + B_1')\dots (\Id + B_{2q}')\mper\nonumber\end{align}


%
	
	\sloppy Last, we deal with the diagonal matrix $K$. Let $K = \diag(k_i)$. We have $\min k_i = \sigma_{\min}(R), \max k_i=\sigma_{\max}(R)$. Then, we can write $K =  (K')^{p}$ where $K' = \diag(k_i^{1/p})$ and $p$ is an integer to be chosen later. We have that $\Norm{K'-\Id} \le \max |k_i^{1/p} -1|\le  \max |e^{\log k_i \cdot 1/p} -1|$. When $p \ge \gamma = \max\{\log \max k_i, - \log \min k_i\} = \max\{\log \sigma_{\max}(R), -\log \sigma_{\min}(R)\}$, we have that 
	\begin{align}
	\Norm{K'-\Id} \le \max |e^{\log k_i \cdot 1/p} -1|\le 3\max |\log k_i\cdot 1/p| = 3\gamma /p\mper\tag{since $|e^{x}-1|\le 3|x|$ for $|x|\le 1$}
	\end{align}
	
	Let $B_1''=\dots = B_p'' =  K'-\Id$ and then we have $K = (\Id + B_p'')\cdots(\Id+B_1'')$. Finally, we choose $p = \frac{3\gamma \ell}{4\pi + 3\gamma}$ and $q = \frac{\pi\ell}{4\pi + 3\gamma}$, \footnote{Here for notational convenience, $p,q$ are not chosen to be integers. But rounding them to closest integer will change final bound of the norm by small constant factor. }and let $A_{p+4q}=B_{2q}, \dots = A_{p+2q+1} = B_1, A_{p+2q}= B_p'', \dots, A_{2q+1} = B_1'', A_{2q} = B_{2q}', \dots, A_1 = B_1'$. We have that $4q+ p= \ell$ and 
	\begin{align}
	R = UKV^{\top} = (\Id+A_{\ell})\dots (\Id + A_1)\mper\nonumber
	\end{align}
	Moreover, we have $\gnorm{A} \le \max\{\norm{B_j}, \norm{B_j'}.\norm{B_j''}\}\le \max\{\pi/q, 3\gamma/p\}\le \frac{4\pi + 3\gamma}{\ell}$, as desired. 
\end{proof}
\subsection{Proof of Lemma~\ref{lem:gradients}}

We compute the partial gradients by definition. 
	Let $\Delta_{j} \in \R^{d\times d}$ be an infinitesimal change to $A_j$. Using Claim~\ref{claim:1}, consider the Taylor expansion of $f(A_1,\dots, A_\ell+\Delta_j, \dots, A_{\ell})$ 
	\begin{align}
	& f(A_1,\dots, A_\ell+\Delta_j, \dots, A_{\ell})\nonumber\\& = \Norm{((\Id+A_{\ell})\cdots(\Id+A_j+\Delta_j)\dots (\Id+A_1)-R)\Sigma^{1/2}}_F^2 \nonumber\\
	& = \Norm{((\Id+A_{\ell})\cdots (\Id+A_1)-R)\Sigma^{1/2} + (\Id+A_{\ell})\cdots \Delta_j \dots (\Id+A_1)\Sigma^{1/2}}_F^2 \nonumber\\
	& = \Norm{((\Id+A_{\ell})\cdots (\Id+A_1)-R)\Sigma^{1/2}}_F^2 + \nonumber\\
	& ~~~~~2\inner{((\Id+A_{\ell})\cdots (\Id+A_1)-R)\Sigma^{1/2}, (\Id+A_{\ell})\cdots \Delta_j \dots (\Id+A_1)\Sigma^{1/2}} + O(\norm{\Delta_j}_F^2)\nonumber\\
	& = f(A) + 2\inner{ (\Id+ A_{j+1}^{\top})\dots (\Id + A_{\ell}^{\top})E \Sigma(\Id + A_{1}^{\top})\dots (\Id + A_{j-1}^{\top}) ,\Delta_j} + O(\norm{\Delta_j}_F^2)\mper\nonumber
	\end{align}
	By definition, \sloppy this means that the $\frac{\partial f}{\partial A_j} = 2(\Id+ A_{\ell}^{\top})\dots (\Id + A_{j+1}^{\top})E \Sigma(\Id + A_{j-1}^{\top})\dots (\Id + A_{1}^{\top})$.\qed

\section{Missing Proofs in Section~\ref{sec:representation}}\label{sec:proof:representation}

In this section, we provide the full proof of Theorem~\ref{thm:representation}. We start with the following Lemma that constructs a building block $\cT$ that transform $k$ vectors of an arbitrary sequence of $n$ vectors to any arbitrary set of vectors, and main the value of the others. For better abstraction we use $\alpha\pp{i}$,$\beta\pp{i}$ to denote the sequence of vectors.  

\begin{lemma}\label{lem:building_block}
	Let $S\subset [n]$ be of size $k$. Suppose $\alpha\pp{1},\dots, \alpha\pp{n}$ is a sequences of $n$ vectors satisfying a) for every $1\le i\le  n$, we have $1-\rho'\le \norm{\alpha_i}^{2} \le 1+\rho'$, and b) if $i\neq j$ and $S$ contains at least one of $i,j$, then  $\norm{\alpha\pp{i}-\alpha\pp{j}}\ge  3\rho'$. Let $\beta\pp{1},\dots, \beta\pp{n}$ be an arbitrary sequence of vectors. Then, there exists $U,V\in \R^{k\times k}, s$ such that  for every $i\in S$, we have $\cT_{U,V, s}(\alpha\pp{i}) = \beta\pp{i} -\alpha\pp{i}$, and moreover,  for every $i\in [n]\backslash S$ we have $\cT_{U,V, s}(\alpha\pp{i}) = 0$. 
\end{lemma}

We can see that the conclusion implies 
\begin{align}
\beta\pp{i} & = \alpha\pp{i} + \cT_{U,V,s}(\alpha\pp{i})~~\forall i\in S \nonumber\\
\alpha\pp{i} & = \alpha\pp{i} + \cT_{U,V,s}(\alpha\pp{i})~~ \forall i\not\in S \nonumber
\end{align}
which is a different way of writing equation~\eqref{eqn:14}.

\begin{proof}[Proof of Lemma~\ref{lem:building_block}]
	Without loss of generality, suppose $S = \{1,\dots, k\}$. We construct $U,V,s$ as follows. Let the $i$-th row of $U$ be  $\alpha\pp{i}$ for $i\in [k]$, and let $s = -(1-2\rho')\cdot \mathbf{1}$ where $\mathbf{1}$ denotes the all 1's vector. Let the $i$-column of $V$ be $\frac{1}{\norm{\alpha\pp{i}}^2-(1-2\rho')} (\beta\pp{i}-\alpha\pp{i})$ for $i\in [k]$. 
	
	Next we verify that the correctness of the construction.  We first consider $1\le i \le k$.  We have that $U\alpha\pp{i}$ is a a vector with $i$-th coordinate equal to $\norm{\alpha\pp{i}}^2\ge 1-\rho'$. The $j$-th coordinate of $U\alpha\pp{i}$ is equal to $\inner{\alpha\pp{j},\alpha\pp{i}}$, which can be upperbounded using the assumption of the Lemma by 
	\begin{align}
	\inner{\alpha\pp{j},\alpha\pp{i}} = \frac{1}{2}\left(\norm{\alpha\pp{i}}^2 + \norm{\alpha\pp{j}}^2\right) - \norm{\alpha\pp{i}-\alpha\pp{j}}^2 \le 1+\rho' - 3\rho' \le 1-2\rho'\mper\label{eqn:12}
	\end{align}
	Therefore, this means $U\alpha\pp{i} - (1-2\rho')\cdot\allones$contains a single positive entry (with value at least $\norm{\alpha\pp{i}}^2-(1-2\rho')\ge \rho'$), and all other entries being non-positive. This means that 
		$\relu(U\alpha\pp{i}+b) = \left(\norm{\alpha\pp{i}}^2-(1-2\rho')\right) e_i $ where $e_i$ is the $i$-th natural basis vector. It follows that $V\relu(U\alpha\pp{i}+b) = (\norm{\alpha\pp{i}}^2-(1-2\rho'))Ve_i= \beta\pp{i}-\alpha\pp{i}$. 
	
	Finally, consider $n\ge i > k$. Then similarly to the computation in equation~\eqref{eqn:12}, $U\alpha\pp{i}$ is a vector with all coordinates less than $1-2\rho'$. Therefore $U\alpha\pp{i} + b$ is a vector with negative entries. Hence we have $\relu(U\alpha\pp{i} + b) = 0$, which implies $V\relu(U\alpha\pp{i}+b) = 0$. 
\end{proof}
Now we are ready to state the formal version of Lemma~\ref{lem:induction_informal}. 

\begin{lemma}\label{lem:induction}
	Suppose a sequence of $n$ vectors $z\pp{1},\dots, z\pp{n}$ satisfies a relaxed version of Assumption~\ref{ass:data}: a) for every $i$, $1-\rho'\le \norm{z\pp{i}}^2 \le 1+\rho'$ b) for every $i\neq j$, we have $\norm{z\pp{i}-z\pp{j}}^2\ge \rho';$. Let $v\pp{1},\dots, v\pp{n}$  be defined above. 
	Then there exists weigh matrices $(A_1,B_1),\dots, (A_{\ell},B_{\ell})$, such that given $\forall i, h_0\pp{i} = z\pp{i}$, we have, 
	\vspace{-.05in}
	\begin{align}
	\forall i\in \{1,\dots, n\}	, ~~~ h_{\ell}\pp{i} = v\pp{i}\mper\nonumber
	\end{align}
\end{lemma}

We will use Lemma~\ref{lem:building_block} repeatedly  to construct building blocks $\cT_{A_j,B_k,s_j}(\cdot)$, and thus prove Lemma~\ref{lem:induction}. Each building block $\cT_{A_j,B_k,s_j}(\cdot)$ takes a subset of $k$ vectors among $\{z\pp{1},\dots, z\pp{n}\}$ and convert them to $v\pp{i}$'s, while maintaining all other vectors as fixed. Since they are totally $n/k$ layers, we finally maps all the $z\pp{i}$'s  to the target vectors $v\pp{i}$'s. 
\begin{proof}[Proof of Lemma~\ref{lem:induction}]
		We use Lemma~\ref{lem:building_block} repeatedly. Let $S_1 = [1,\dots, k]$. Then using Lemma~\ref{lem:building_block} with $\alpha\pp{i} = z\pp{i}$ and $\beta\pp{i} =  v\pp{i}$ for $i\in [n]$, we obtain that there exists $A_1,B_1,b_1$ such that for $i\le k$, it holds that $h_1\pp{i} = z\pp{i} + \cT_{A_1,B_1,b_1}(z\pp{i}) = v\pp{i}$, and for $i\ge k$, it holds that $h_1\pp{i} = z\pp{i} + \cT_{A_1,B_1,b_1}(z\pp{i}) = z\pp{i}$. 
	
	Now we construct the other layers inductively. We will construct the layers such that the hidden variable at layer $j$ satisfies $h_{j}\pp{i} = v\pp{i}$ for every $1\le i\le jk$, and $h_j\pp{i} = z\pp{i}$ for every $n\ge i> jk$. Assume that we have constructed the first $j$ layer and next we use Lemma~\ref{lem:building_block} to construct the $j+1$ layer. Then we argue that the choice of $\alpha\pp{1} = v\pp{1}, \dots, \alpha\pp{jk} = v\pp{jk}$, $\alpha\pp{jk+1} = z\pp{jk+1},\dots, \alpha\pp{n} = z\pp{n}$, and $S = \{jk+1,\dots, (j+1)k\}$ satisfies the assumption of Lemma~\ref{lem:building_block}. Indeed, because $q_i$'s are chosen uniformly randomly, we have w.h.p for every $s$ and $i$, $\inner{q_s, z\pp{i}}\le 1-\rho'$. Thus, since $v\pp{i}\in \{q_1,\dots, q_r\}$, we have that $v\pp{i}$ also doesn't correlate with any of the $z\pp{i}$. Then we apply Lemma~\ref{lem:building_block} and conclude that there exists $A_{j+1} = U, B_{j+1} = V, b_{j+1} = s$ such that $\cT_{A_{j+1},b_{j+1}, b_{j+1}}(v\pp{i}) = 0$ for $i\le jk$, $\cT_{A_{j+1},b_{j+1}, b_{j+1}}(z\pp{i}) = v\pp{i}-z\pp{i}$ for $jk < i \le (j+1)k$, and $\cT_{A_{j+1},b_{j+1}, b_{j+1}}(z\pp{i}) = 0$ for $n\ge  i > (j+1)k$. These imply that 
	\begin{align}
	h_{j+1}^{\pp{i}} & = h_j\pp{i} + \cT_{A_{j+1},b_{j+1}, b_{j+1}}(v\pp{i}) = v\pp{i} \quad \forall 1\le i \le jk\nonumber\\
	h_{j+1}^{\pp{i}} & = h_j\pp{i} + \cT_{A_{j+1},b_{j+1}, b_{j+1}}(z\pp{i}) = v\pp{i} \quad \forall jk+1\le i \le (j+1)k \nonumber\\
	h_{j+1}^{\pp{i}} & = h_j\pp{i} + \cT_{A_{j+1},b_{j+1}, b_{j+1}}(z\pp{i}) = z\pp{i}\quad \forall (j+1)k <  i \le n \nonumber
	\end{align}
	Therefore we constructed the $j+1$ layers that meets the inductive hypothesis for layer $j+1$. Therefore, by induction we get all the layers, and the last layer satisfies that $ h_{\ell}\pp{i}= v\pp{i}$ for every example $i$. 	\end{proof}
Now we ready to prove Theorem~\ref{thm:representation}, following the general plan sketched in Section~\ref{sec:representation}. 
\begin{proof}[Proof of Theorem~\ref{thm:representation}]
	We use formalize the intuition discussed below Theorem~\ref{thm:representation}. First, take $k= c(\log n)/\rho^2$ for sufficiently large absolute constant $c$ (for example, $c=10$ works), by Johnson-Lindenstrauss Theorem (\cite{johnson1984extensions}, or see~\cite{wiki:JL}) we have that when $A_0$ is a random matrix with standard normal entires, with high probability, all the pairwise distance between the the set of vectors $\{0, x\pp{1},\dots, x\pp{n}\}$ are preserved up to $1\pm \rho/3$ factor. That is, we have that for every $i$, $1-\rho/3 \le \norm{A_0x\pp{i}} \le 1+\rho /3$, and for every $i\neq j$, $\norm{A_0x\pp{i}-A_0x\pp{j}}\ge \rho (1-\rho/3)\ge 2\rho /3$. Let $z\pp{i}= A_0x\pp{i}$ and $\rho' = \rho/3$. Then we have $z\pp{i}$'s satisfy the condition of Lemam~\ref{lem:induction}. We pick $r$ random vectors $q_1, \dots, q_r$ in $\R^k$.   Let $v\pp{1},\dots, v\pp{n}$ be defined as in equation~\eqref{eqn:def-v}.  	
	Then by Lemma~\ref{lem:induction}, we can construct matrices $(A_1,B_1),\dots, (A_{\ell},B_{\ell})$ such that 
	\begin{align}
	h_{\ell}\pp{i} = v\pp{i}\mper\label{eqn:13}
	\end{align}	
	Note that $v\pp{i}\in \{q_1,\dots, q_r\}$, and $q_i$'s are random unit vector. Therefore, the choice of $\alpha\pp{1}=q_1,\dots, \alpha\pp{r} = q_r$, $\beta\pp{1}=e_1,\dots, \beta\pp{r}=e_r$, and satisfies the condition of Lemma~\ref{lem:building_block}, and using Lemma~\ref{lem:building_block} we conclude that there exists $A_{\ell+1}, B_{\ell+1}, s_{\ell+1}$ such that 
	\begin{align}
	e_j =  \cT_{A_{\ell+1},B_{\ell+1},b_{\ell+1}}(v_j), \textup{ for every } j\in \{1,\dots, r\}\mper. 
	\end{align}
	By the definition of $v\pp{i}$ in equation~\eqref{eqn:def-v} and equation~\eqref{eqn:13}, we conclude that $\hat{y}\pp{i} = h_{\ell}\pp{i} +  \cT_{A_{\ell+1},B_{\ell+1},b_{\ell+1}}(h_{\ell}\pp{i}) = y\pp{i}. $, which complete the proof.
\end{proof}

\section{Toolbox}

In this section, we state two folklore linear algebra statements. The following Claim should be known, but we can't find it in the literature. We provide the proof here for completeness. 
\begin{claim} \label{cliam:block-diagonal}Let $U\in \R^{d\times d}$ be a real normal matrix (that is, it satisfies $UU^{\top} = U^{\top}U$).  Then, there exists an orthonormal matrix $S\in \R^{d\times d}$ such that 
	\begin{align}
		U = SDS^{\top}\mcom \nonumber
	\end{align}
	where $D$ is a real block diagonal matrix that consists of blocks with size at most $2\times 2$. 
\end{claim}

\Tnote{May need to expand this proof. }
\begin{proof}
	Since $U$ is a normal matrix, it is unitarily diagonalizable (see ~\cite{normalmatrix} for backgrounds). Therefore, there exists unitary matrix $V$ in $\C^{d\times d}$ and diagonal matrix in $\C^{d\times d}$ such that $U$ has eigen-decomposition $U =  V\Lambda V^*$. Since $U$ itself is a real matrix, we have that the eigenvalues (the diagonal entries of $\Lambda$) come as conjugate pairs, and so do the eigenvectors (which are the columns of $V$). That is, we can group the columns of $V$ into pairs $(v_1,\bar{v}_1),\dots, (v_{s},\bar{v_s}), v_{s+1},\dots, v_{t}$, and let the corresponding eigenvalues be $\lambda_1,\bar{\lambda}_1, \dots, \lambda_{\lambda_s},\bar{\lambda}_s, \lambda_{s+1},\dots, \lambda_t$. Here $\lambda_{s+1},\dots, \lambda_t\in \R$. Then we get that $U =\sum_{i=1}^s 2\Re(v_i\lambda_i v_i^*) + \sum_{i=s+1}^t v_i\lambda_iv_i^{\top}$. Let $Q_i = \Re(v_i\lambda_i v_i^*)$, then we have that $Q_i$ is a real matrix of rank-2. Let $S_i\in \R^{d\times 2}$ be a orthonormal basis of the column span of $Q_i$ and then we have that $Q_i$ can be written as $Q_i = S_iD_iS_i^{\top}$ where $D_i$ is a $2\times 2$ matrix. Finally, let $S =[S_1,\dots, S_s,v_{s+1},\dots, v_t]$, and $D = \diag(D_1,\dots, D_s,\lambda_{s+1},\dots, \lambda_{t})$ we complete the proof. 
\end{proof}

The following Claim is used in the proof of Theorem~\ref{thm:main}. We provide a proof here for completeness. 
\begin{claim}[folklore]\label{claim:sigmamin}	For any two matrices $A,B\in \R^{d\times d}$, we have that 
	\begin{align}
	\norm{AB}_F \ge \sigma_{\min}(A) \norm{B}_F \mper\nonumber
	\end{align}
\end{claim}
\begin{proof}
	Since $\sigma_{\min}(A)^2$ is the smallest eigenvalue of $A^{\top}A$, we have that 
	\begin{align}
	B^{\top}A^{\top}AB \succeq B^{\top}\cdot \sigma_{\min}(A)^2\Id \cdot B\mper\nonumber
	\end{align}
	Therefore, it follows that 
	\begin{align}
		\norm{AB}_F^2 & =\trace(B^{\top}A^{\top}AB) \ge \trace(B^{\top}\cdot \sigma_{\min}(A)^2\Id \cdot B)\nonumber \\
		& = \sigma_{\min}(A)^2 \trace(B^{\top}B) = \sigma_{\min}(A)^2\norm{B}_F^2\mper\nonumber
	\end{align}
	Taking square root of both sides completes the proof. 
\end{proof}

\bibliographystyle{plain}

\end{document}